\documentclass[twoside]{article}

\usepackage[accepted]{aistats2023}
\usepackage[round]{natbib}

\usepackage{amsfonts}
\usepackage{amsmath}
\usepackage{mathrsfs}
\usepackage{amsthm}

\newtheorem{theorem}{Theorem}
\newtheorem{definition}[theorem]{Definition}
\newtheorem{remark}[theorem]{Remark}
\newtheorem{lemma}[theorem]{Lemma}

\newtheorem{proposition}[theorem]{Proposition}

\begin{document}

%

%

\twocolumn[

\aistatstitle{Universal Agent Mixtures and the Geometry of Intelligence}
\aistatsauthor{Samuel Allen Alexander \And David Quarel}
\aistatsaddress{The U.S.\ Securities and Exchange Commission \And Australian National University}
\aistatsauthor{Len Du \And Marcus Hutter}
\aistatsaddress{WooliesX \And DeepMind \& Australian National University}
\runningtitle{Universal Agent Mixtures}
\runningauthor{Alexander, Du, Quarel, Hutter}
]

\begin{abstract}
    Inspired by recent progress in multi-agent Reinforcement Learning (RL),
    in this work we examine the collective intelligent behaviour
    of theoretical universal agents by
    introducing a weighted mixture operation.
    Given a weighted set of agents,
    their weighted mixture is a new agent whose expected total reward in any environment
    is the corresponding weighted average
    of the original agents' expected total rewards in that environment.
    Thus, if RL agent intelligence is quantified in terms of performance across
    environments, the weighted mixture's intelligence is the weighted average of
    the original agents' intelligences.
    This operation enables
    various interesting new theorems that shed light on the geometry of RL
    agent intelligence, namely: results about symmetries, convex agent-sets,
    and local extrema. We also show that any RL agent intelligence measure
    based on average performance across environments, subject to certain
    weak technical conditions, is identical (up to a constant factor) to
    performance within a single environment dependent on said intelligence
    measure.
\end{abstract}

\section{INTRODUCTION}

Multi-agent Reinforcement Learning (or multi-agent RL) \citep{collective-1993,multiagent-1994-most-cited,multiagent-2021-most-cited,hernandez2011more}, as with other flavors of RL, has been enjoying increased attention in artificial intelligence research \citep{unified-multiagent}.
The most obvious way to conceive multi-agent RL, is to passively consider the collective behavior of aggregated intelligent agents.
In fact, multi-agent RL was first introduced as ``collective learning'' \citep{collective-1993} even before being explicitly identified as multi-agent reinforcement learning \citep{multiagent-1994-most-cited}.


One significant recent trend in machine learning (beyond just RL)
is Federated Reinforcement Learning \citep{federated-learning-2021},
which is mainly about programs physically running on disparate devices collaborating to form a more powerful artificial intelligence. This approach conceptually borrows from how humans collaborate.
In this regard, RL had a much earlier head start with
Feudal Reinforcement Learning \citep{feudal-1992}, which borrows concepts which have been around for many hundreds of years in human social organization, and on which research remains active today \citep{Johnson2020FeudalSH}.
Multi-agent methods have also received considerable attention for their usage
in highly complex real-time video-games \citep{nature-starcraft}
(building off of success in simpler games like Atari games
\citep{nature-ATARI} and Go \citep{nature-go-2}).
RL is no stranger to
collaboration \citep{Kok2006CollaborativeMR} or to
cooperation \citep{Qiu2021RMIXLR}, including even collaboration with
human users \citep{collab-other}.

In this work, we draw inspiration from \emph{sortition}, which is yet another way of social organization. In sortition, instead of the whole citizenry collaborating on individual
decisions, citizens are chosen \emph{by lottery} and granted temporary power. Thus in a
statistical sense, each citizen enjoys a certain amount of total expected power.
The roots of sortition trace back to the original Athenian democracy \citep{sortition-ancient},
and sortition has attracted recent scientific curiosity, and even advocacy in real-world governance \citep{nature-sortition,Sintomer2018FromDT,sortition-2013}.

More specifically, we examine the expected intelligence of a single combined
mixture agent formed from a group of agents using sortition.
Given agents $\vec{\pi}=(\pi_1,\ldots,\pi_n)$,
imagine an agent $\sigma$
who, at the start of each agent-environment interaction, randomly chooses
an agent $\pi_i$ to act as for that entire interaction (that is, the selection
only occurs \emph{once}, at the beginning, and then persists---we do \emph{not}
mean that one of the $\pi_i$ is randomly chosen on every single turn).
Imagine that each candidate $\pi_i$ is so chosen with probability (or \emph{weight})
$w_i$ (where $w_1+\cdots+w_n=1$). If each $\pi_i$ would get total expected
reward $R_i$ from an environment,
we would expect $\sigma$ to get total expected reward
$\vec{w}\cdot \vec{R}=w_1R_1+\cdots+w_nR_n$ from that environment.

Starting from practical (e.g., conformant with OpenAI Gym \citep{brockman2016openai})
implementations of agents $\pi_1,\ldots,\pi_n$ as above, $\sigma$ could easily be
implemented as a new agent who, upon instantiation, uses a random number generator
to determine which $\pi_i$ to act as, and stores that decision in internal memory.
But this sort of construction is not possible in more abstract, theoretical
RL frameworks such as \citep{legg2007universal, DBLP:conf/ijcai/LeggH05}, where agents are mathematical functions
which take histories as input and output action-space probability-distributions.
Such functions have no ``instantiation'', no access to a true random number
generator, and no concept of ``internal memory''. A key result of ours is that
nevertheless, it is possible to define an agent $\sigma$ having the same exact performance
as the above-described sortition agent, entirely within an abstract, theoretical
RL framework. This is important because the constraints of the theoretical framework
facilitate rigorous mathematical proofs of properties of $\sigma$. To see that the
construction
is non-trivial, consider instead an agent $\rho$ who, upon instantiation, examines the
computer's system clock, and determines to play as $\pi_1$ if the clock says ``AM'' or as
$\pi_2$ if the clock says ``PM''. Such a $\rho$ could certainly be implemented in an
OpenAI gym conformant way, but clearly has no counterpart in a formal RL framework
with no notion of a system clock.

The fact that the weighted mixture agent $\sigma$'s expected total reward in any
environment is the corresponding weighted average of the expected total rewards of
$\pi_1,\ldots,\pi_n$ will allow us to prove multiple interesting results that
shed light on the geometry of RL agent performance and performance-based intelligence
measures. We obtain the following results and applications:
\begin{itemize}
    \item (Section \ref{mixtureagentsection})
    By guaranteeing that ``the expected reward of a weighted mixture is the
    weighted average of the expected rewards'', we establish a method of combining
    agents without the risk of unforeseen side-effects.
    For example, if several agents have different weaknesses, then, a priori,
    one might worry that, combining those agents, those weaknesses might compound
    each other, leading to a combined weakness larger than the sum of the individual
    weaknesses. Our mixture agent construction avoids this, as well as other emergent
    behavior which would violate the above quote.
    \item (Section \ref{symmetrysection})
    We consider two different ways an intelligence measure can be symmetric
    with respect to the operation of interchanging rewards and punishments.
    We prove that these two symmetry notions are equivalent.
    This has implications in the search for inherently desirable properties of
    universal Turing machines.
    \item (Section \ref{convexitysection})
    We introduce notions of \emph{discernability} and \emph{separability} of
    sets of RL agents, and characterize the latter in terms of the former and
    closure under our mixture operation. If agents are thought of as points in
    space, then these properties are analogous to
    higher-dimensional convexity notions from convex geometry.
    These results can help determine what sort of things can or cannot
    be incentivized in RL, in a formal sense (similar to using the Pumping Lemma
    to show that certain languages are not regular).
    \item (Section \ref{extremasection})
    We introduce a notion of an agent being a strict local extremum of an
    intelligence measure, and we show that any such agent is, in a certain
    formal sense, deterministic. This result is highly applicable in the quest
    to optimize RL agents, as it implies that nothing is gained by allowing
    agents to invoke genuine random number generators, e.g.\ expensive RNGs
    based on quantum mechanics, etc.
    \item (Section \ref{mixtureenvsection}) Finally, we use our technique to
    mix environments, rather than agents. Using the resulting mixture environments,
    we prove that every intelligence measure satisfying certain properties is
    necessarily equivalent (up to a constant multiple) to performance in some
    particular environment.
\end{itemize}

\section{PRELIMINARIES}

Throughout the paper, we implicitly
fix non-empty finite sets $\mathcal A$ of \emph{actions},
$\mathcal O$ of \emph{observations},
and $\mathcal R\subseteq \mathbb Q\cap [-1,1]$ of \emph{rewards}.
By $\varepsilon$ we mean the empty sequence.
By $\mathcal E$ we mean $\mathcal O\times\mathcal R$ (the set of all observation-reward
pairs); elements of $\mathcal E$ are called \emph{percepts}.
By $\Delta\mathcal A$ (resp.\ $\Delta\mathcal E$) we mean the set of all
probability distributions on $\mathcal A$ (resp.\ on $\mathcal E$).

\begin{definition}
\label{omnibusdefn}
    (Agents, environments, etc.)
    \begin{enumerate}
        \item
        We denote the set of all finite sequences
        of alternating percept-action pairs $x_1y_1\ldots x_ty_t$
        by $(\mathcal E\mathcal A)^*$.
        We also include $\varepsilon$ in $(\mathcal E\mathcal A)^*$.
        Nonempty elements of $(\mathcal E\mathcal A)^*$ have the
        form $x_1y_1\ldots x_ty_t$ where each $x_i$ is a percept and
        each $y_i$ is an action.
        \item
        We denote the set of all sequences of the form $sx$ (where
        $s\in (\mathcal E\mathcal A)^*$, $x\in\mathcal E$, and $sx$
        is the result of appending $x$ to $s$) by
        $(\mathcal E\mathcal A)^*\mathcal E$.
        Elements of $(\mathcal E\mathcal A)^* \mathcal E$
        of length $>1$ have the form
        $x_1y_1\ldots x_{t-1}y_{t-1}x_t$
        (each $x_i$ a percept, each $y_i$ an action).
        \item
        An \emph{agent}\footnote{Not to be confused with a \emph{policy}, which
        would simply be a function $\mathcal O\to\Delta\mathcal A$.}
        is a function
        $\pi:(\mathcal E\mathcal A)^*\mathcal E\to \Delta \mathcal A$.
        For any $h\in (\mathcal E\mathcal A)^*\mathcal E$,
        we write $\pi(\cdot|h)$ for the value of $\pi$ at $h$, and
        for any $y\in \mathcal A$, we write $\pi(y|h)$ for
        $(\pi(\cdot|h))(y)$.
        Intuitively, for any action $y$,
        $\pi(y|h)$ is the probability that agent $\pi$
        takes action $y$ in response to history $h$.
        \item
        An \emph{environment} is a function
        $\mu:(\mathcal E\mathcal A)^*\to\Delta\mathcal E$.
        For every $h\in(\mathcal E\mathcal A)^*$, we write
        $\mu(\cdot|h)$ for the value of $\mu$ at $h$, and for any
        $x\in\mathcal E$, we write $\mu(x|h)$ for $(\mu(\cdot|h))(x)$.
        If $x=(o,r)$ ($o\in\mathcal O$, $r\in\mathcal R$), we may also
        write $\mu(o,r|h)$ for $(\mu(\cdot|h))(x)$.
        Intuitively, $\mu(o,r|h)$ is the probability that environment
        $\mu$ issues percept $(o,r)$ (observation $o$ and reward $r$)
        to the agent in response to history $h$.
    \end{enumerate}
\end{definition}

\begin{remark}
\label{impossibleremark}
    Note that in Definition \ref{omnibusdefn} part 3, we require,
    e.g., $\pi(\cdot|x_1y_1x_2)$ to be defined even if
    $\pi(y_1|x_1)=0$, in which case the initial percept-action sequence $x_1y_1x_2$
    would have probability $0$ of ever occurring in any agent-environment
    interaction. Intuitively: an agent must choose actions even
    in response to histories that would never occur with nonzero probability.
    This convention, in which we follow \citet{legg2007universal}, simplifies
    many definitions.
\end{remark}

\begin{definition}
    By $\mathcal H$ we mean
    $((\mathcal E\mathcal A)^*)\cup((\mathcal E\mathcal A)^*\mathcal E)$,
    in other words, $\mathcal H$ is the set of alternating percept-action
    sequences that are empty or else start with a percept and can end with
    either a percept or an action.
    We refer to elements $h$ of $\mathcal H$ as \emph{histories} (a history
    may terminate with either a percept or an action).
\end{definition}

\begin{definition}
\label{pullbackdef}
    For all agents $\pi$, histories $h$, and environments $\mu$,
    we define real numbers $P^\pi(h)$, $P_\mu(h)$, and $P^\pi_\mu(h)$
    inductively as follows.
    \begin{itemize}
        \item
        If $h=\varepsilon$ then $P^\pi(h)=P_\mu(h)=P^\pi_\mu(h)=1$.
        \item
        If $h=gx$ (some $x\in\mathcal E$) then
        $P^\pi(h)=P^\pi(g)$, $P_\mu(h)=P_\mu(g)\mu(x|g)$,
        and $P^\pi_\mu(h)=P^\pi_\mu(g)\mu(x|g)$.
        \item
        If $h=gy$ (some $y\in\mathcal A$) then
        $P^\pi(h)=P^\pi(g)\pi(y|g)$, $P_\mu(h)=P_\mu(g)$,
        and $P^\pi_\mu(h)=P^\pi_\mu(g)\pi(y|g)$.
    \end{itemize}
    Intuitively: $P^\pi(h)$ is the conditional
    probability $\pi$ will choose the actions in $h$ assuming the
    environment which $\pi$ is interacting with chooses the percepts in
    $h$; $P_\mu(h)$ is the conditional probability $\mu$ will choose
    the percepts in $h$ assuming the agent which $\mu$ is interacting
    with chooses the actions in $h$; and $P^\pi_\mu(h)$ is the probability
    that $\pi$ and $\mu$ will choose $h$'s actions and percepts when
    interacting together.
\end{definition}

Some authors, such as \citet{hutter2009discrete}, would write $P(h)$ or a variation thereof
for $P^\pi_\mu$, if $\pi$ and $\mu$ are clear from context.

One could alternately more directly define
\begin{align*}
    {} & P^\pi(x_1y_1\ldots x_ty_t)\\
    &= \pi(y_1|x_1)\pi(y_2|x_1y_1x_2)\cdots \pi(y_t|x_1y_1\ldots x_t),
\end{align*}
and similarly define $P^\pi(x_1y_1\ldots x_t)$,
and likewise for $P_\mu$ and for $P^\pi_\mu$.

\begin{lemma}
\label{factorizationlemma}
    For all $h$, $\pi$, $\mu$ as in Definition \ref{pullbackdef},
    \[
        P^\pi_\mu(h) = P^\pi(h)P_\mu(h).
    \]
\end{lemma}

\begin{proof}
    See Supplementary Materials.
\end{proof}

In the following definition (and the rest of the paper),
$\mathbb N$ denotes the set $\{0,1,2,\ldots\}$ of non-negative
integers.

\begin{definition}
\label{performancedefn}
    (Performance in an environment)
    Let $\pi$ be an agent, $\mu$ an environment.
    \begin{enumerate}
    \item
        For every $t\in\mathbb N$,
        we define
        \[
            V^\pi_{\mu,t}=\sum_{h\in X_t}R(h)P^\pi_\mu(h)
        \]
        where $X_t\subseteq\mathcal H$ is the set of all
        length-$2t$ histories (i.e., all $h\in\mathcal H$ of the form
        $x_1y_1\ldots x_ty_t$ (each $x_i\in\mathcal E$, each $y_i\in\mathcal A$)
        provided $t>0$) and $R(h)$ is the sum of the rewards in $h$.
        Intuitively, $V^\pi_{\mu,t}$ is the expected total reward
        if $\pi$ were to interact with $\mu$ for $t$ steps.
        Note that $X_0=\{\varepsilon\}$ and so $V^\pi_{\mu,0}=0$.
    \item
        We define $V^\pi_\mu=\lim_{t\to\infty}V^\pi_{\mu,t}$,
        provided the limit converges to a real number.
        Intuitively, $V^\pi_\mu$ is the expected total reward which $\pi$ would extract
        from $\mu$.
    \end{enumerate}
\end{definition}

Note that it is possible for $V^\pi_\mu$ to be undefined.
For example, if $\mu$ is an environment which always issues
reward $(-1)^t$ in response to the agent's $t$th action $y_t$,
then $V^\pi_\mu$ is undefined for every agent $\pi$.
We will only be interested in
environments $\mu$ such that $V^\pi_\mu$
is always defined. Note also that, following \citet{legg2007universal},
we delegate any possible reward discounting to the environments themselves,
rather than build a fixed reward discounting factor into the definition
of $V^\pi_\mu$.

\begin{definition}
\label{wellbehaveddefn}
    An environment $\mu$ is \emph{well-behaved} if the following
    requirements hold: $\mu(x|h)\in\mathbb Q$
    for all $x\in\mathcal E$ and $h\in (\mathcal E\mathcal A)^*$; $\mu$ is Turing
    computable; and for every agent $\pi$, $V^\pi_\mu$ exists and
    $-1\leq V^\pi_\mu\leq 1$. Let $W$ be the set of all well-behaved environments.
\end{definition}

\begin{definition}
\label{performanceaveragerdefn}
    By a \emph{weighted intelligence measure}, we mean a function
    $
        \Upsilon:
        (\Delta\mathcal A)^{(\mathcal E\mathcal A)^*\mathcal E}
        \to
        \mathbb R
    $
    (where
    $(\Delta\mathcal A)^{(\mathcal E\mathcal A)^*\mathcal E}$ denotes the set of all agents)
    such that there exist non-negative reals $\{w_\mu\}_{\mu\in W}$ such that the following
    condition holds:
    for every agent $\pi$, $\Upsilon(\pi)=\sum_{\mu\in W}w_\mu V^\pi_\mu$.
\end{definition}

The prototypical weighted intelligence measure is the Legg-Hutter intelligence
measure $\Upsilon$ introduced by \citet{legg2007universal},
where each well-behaved $\mu$ is
weighed using the universal prior \citep{li2008introduction,DBLP:conf/alt/Hutter03}, i.e.,
given weight $2^{-K(\mu)}$
where $K$ denotes Kolmogorov complexity ($K(\mu)$ exists because of the Turing
computability requirement in Definition \ref{wellbehaveddefn}).
This depends on a background universal
Turing machine, the choice of which is highly nontrivial
\citep{leike2015bad}.

\section{MIXTURE AGENTS}
\label{mixtureagentsection}

Before defining mixture agents, we will first extend some of the above
definitions to vectors of agents.

\begin{definition}
\label{vectorizationdefn}
    Suppose $\vec\pi=(\pi_1,\ldots,\pi_n)$ is a vector of agents, $\mu$ is an environment,
    $h\in\mathcal H$, $t\in\mathbb N$, and $\Upsilon$ is a weighted
    intelligence measure. We define:
    \begin{itemize}
        \item ${P^{\vec\pi}}(h)=(P^{\pi_1}(h),\ldots,P^{\pi_n}(h))$.
        \item $P^{\vec\pi}_\mu(h)=(P^{\pi_1}_\mu(h),\ldots,P^{\pi_n}_\mu(h))$.
        \item $V^{\vec\pi}_{\mu,t}=(V^{\pi_1}_{\mu,t},\ldots,V^{\pi_n}_{\mu,t})$.
        \item $V^{\vec\pi}_\mu=(V^{\pi_1}_\mu,\ldots,V^{\pi_n}_\mu)$,
            if $V^{\pi_1}_\mu,\ldots,V^{\pi_n}_\mu$ are defined.
        \item $\Upsilon(\vec\pi)=(\Upsilon(\pi_1),\ldots,\Upsilon(\pi_n))$.
    \end{itemize}
\end{definition}

\begin{definition}
\label{dotproductdefn}
    If $\vec u=(u_1,\ldots,u_n)$ and
    $\vec v=(v_1,\ldots,v_n)$ are any two equal-length
    vectors of real numbers, then their \emph{dot product}
    is defined to be $\vec u\cdot \vec v=u_1v_1+\cdots+u_nv_n$.
\end{definition}

Now we are ready to define mixture agents. As a motivating example, suppose
we want to combine two agents into a joint agent, but we are worried that
doing so might lead to unexpected emergent behavior. For example, we do not want the
two agents' weaknesses to compound each other and give rise to a joint weakness
larger than the sum of the two agents' weaknesses. We will show below that the following
construction avoids such unexpected behavior.

\begin{definition}
\label{maindefn}
    (Mixture agents)
    Suppose $\vec\pi=(\pi_1,\ldots,\pi_n)$ are agents and $\vec w=(w_1,\ldots,w_n)$
    are positive real numbers with $w_1+\cdots+w_n=1$.
    Define the \emph{mixture agent} $\vec w\cdot\vec\pi$ as follows: for all
    $h\in (\mathcal E\mathcal A)^*\mathcal E$, $y\in\mathcal A$, let
    \[
        (\vec w\cdot\vec\pi)(y|h)
        =
        \begin{cases}
            \dfrac{\vec w\cdot {P^{\vec\pi}}(hy)}{\vec w\cdot {P^{\vec\pi}}(h)}
            &\mbox{if $\vec w\cdot {P^{\vec\pi}}(h)\not=0$,}\\
            1/|\mathcal{A}| &\mbox{otherwise.}
        \end{cases}
    \]
\end{definition}

\begin{remark}
    Definition \ref{maindefn} is complicated by the way
    Definition \ref{omnibusdefn} part 3
    forces us to include the case
    $(\vec w\cdot\vec\pi)(y|h)=1/|\mathcal A|$ when
    $\vec w\cdot {P^{\vec\pi}}(h)=0$
    (see Remark \ref{impossibleremark}): we are obligated to specify how
    $\vec w\cdot\vec\pi$ chooses actions even in response to ``impossible''
    histories for $\vec w\cdot\vec\pi$ (histories containing actions
    $\vec w\cdot\vec\pi$ would never take in those circumstances).
\end{remark}

Intuitively, $\vec w\cdot\vec\pi$ can be thought of as being driven by an entity
who believes that actions are to be chosen by one of the $\vec\pi$, but does not
know which. The entitiy initially assigns each $w_i$ to $\pi_i$ as a prior, and
attempts to guess the probability of each action being chosen by the unknown
agent $\pi_i$, using these priors to do so. As new actions are seen, said priors
are updated using Bayes' rule.

\begin{lemma}
\label{mixturereallyisanagent}
    If $\vec\pi$ and $\vec w$ are as in Definition \ref{maindefn}
    then the mixture agent $\vec w\cdot\vec\pi$ is an agent
    (per Definition \ref{omnibusdefn} part 3).
\end{lemma}

\begin{proof}
    See Supplementary Materials.
\end{proof}

We will frequently use Lemma \ref{mixturereallyisanagent} without explicit mention.
For example, the lemma allows us to speak of $P^{\vec w\cdot \vec\pi}(h)$
(Definition \ref{pullbackdef}), $V^{\vec w\cdot\vec\pi}_{\mu}$
(Definition \ref{performancedefn}), etc., and we will freely do so without
explicitly citing Lemma \ref{mixturereallyisanagent}.

\begin{theorem}
\label{maintheorem}
    (Commutativity of $\vec w$)
    Let $\vec\pi=(\pi_1,\ldots,\pi_n)$ be agents.
    Let $\vec w=(w_1,\ldots,w_n)$ be positive reals with
    $w_1+\cdots+w_n=1$. Let $\mu$ be any environment.
    Then:
    \begin{enumerate}
        \item
        For any $h\in\mathcal H$,
        $P^{\vec w\cdot \vec\pi}(h)=\vec w\cdot {P^{\vec\pi}}(h)$.
        \item
        For any $h\in\mathcal H$,
        $P^{\vec w\cdot \vec\pi}_\mu(h)=\vec w \cdot P^{\vec\pi}_\mu(h)$.
        \item
        For any $t\in\mathbb N$,
        $V^{\vec w\cdot \vec\pi}_{\mu,t}=\vec w\cdot V^{\vec\pi}_{\mu,t}$.
        \item
        (``The expected reward of a weighted mixture is the weighted
        average of the expected rewards'')
        If $V^{\vec\pi}_\mu$ is defined, then
        $V^{\vec w\cdot\vec\pi}_\mu=\vec w\cdot V^{\vec\pi}_\mu$.
        \item
        (``The intelligence of a weighted mixture is the weighted average
        of the intelligences'')
        For any weighted intelligence measure $\Upsilon$,
        $\Upsilon(\vec w\cdot\vec\pi)=\vec w\cdot\Upsilon(\vec\pi)$.
    \end{enumerate}
\end{theorem}

\begin{proof}
    (1) By induction on $h$.

    Case 1: $h=\varepsilon$. Then
    \[
        P^{\vec w\cdot\vec\pi}(h)=1=w_1\cdot 1+\cdots+w_n\cdot 1
        =
        \vec w\cdot {P^{\vec\pi}}(h).
    \]

    Case 2: $h=gx$ for some $x\in\mathcal E$. Then
    \begin{align*}
        P^{\vec w\cdot\vec\pi}(h)
            &= P^{\vec w\cdot\vec\pi}(g)
                &\mbox{(Definition \ref{pullbackdef})}\\
            &= \vec w\cdot {P^{\vec\pi}}(g)
                &\mbox{(Induction)}\\
            &= \vec w\cdot {P^{\vec\pi}}(gx)= \vec w\cdot{P^{\vec\pi}}(h).
                &\mbox{(Definition \ref{pullbackdef})}
    \end{align*}

    Case 3: $h=gy$ for some $y\in\mathcal A$.

    Subcase 3.1: $P^{\vec w\cdot\vec\pi}(g)=0$.
        By induction $\vec w\cdot {P^{\vec\pi}}(g)=0$.
        Since the $w_i$ are positive, this implies
        each $P^{\pi_i}(g)=0$.
        Thus each
        \begin{align*}
            w_i P^{\pi_i}(gy)
                &= w_i P^{\pi_i}(g)\pi_i(y|g)
                    &\mbox{(Definition \ref{pullbackdef})}\\
                &= 0w_i\pi_i(y|g)=0,
        \end{align*}
        i.e., $\vec w\cdot{P^{\vec\pi}}(gy)=0$.
        And
        \begin{align*}
            P^{\vec w\cdot\vec\pi}(gy)
                &= P^{\vec w\cdot\vec\pi}(g)(\vec w\cdot\vec\pi)(y|g)
                    &\mbox{(Definition \ref{pullbackdef})}\\
                &= 0(\vec w\cdot\vec\pi)(y|g) = 0,
        \end{align*}
        so $P^{\vec w\cdot\vec\pi}(h)=\vec w\cdot P^{\vec\pi}(h)=0$.

    Subcase 3.2: ${P^{\vec w\cdot\vec\pi}}(g)\not=0$. Then
    \begin{align*}
        P^{\vec w\cdot\vec\pi}(h)
            &= P^{\vec w\cdot\vec\pi}(g)(\vec w\cdot\vec\pi)(y|g)
                &\mbox{(Definition \ref{pullbackdef})}\\
            &= P^{\vec w\cdot\vec\pi}(g)
                \frac
                {\vec w\cdot {P^{\vec\pi}}(gy)}
                {\vec w\cdot {P^{\vec\pi}}(g)}
                &\mbox{(Definition \ref{maindefn})}\\
            &= \vec w\cdot {P^{\vec\pi}}(g)
                \frac
                {\vec w\cdot {P^{\vec\pi}}(gy)}
                {\vec w\cdot {P^{\vec\pi}}(g)}
                &\mbox{(Induction)}\\
            &= \vec w\cdot {P^{\vec\pi}}(gy) = \vec w\cdot {P^{\vec\pi}}(h).
                &\mbox{(Basic Algebra)}
    \end{align*}

    (2) Follows from (1) and Lemma \ref{factorizationlemma}.

    (3) With $X_t$ and $R$ as in Definition \ref{performancedefn}, we compute:
    \begin{align*}
        V^{\vec w\cdot\vec\pi}_{\mu,t}
            &= \mbox{$\sum_{h\in X_t}R(h)P^{\vec w\cdot\vec\pi}_\mu(h)$}
                &\mbox{(Def.\ \ref{performancedefn})}\\
            &= \mbox{$\sum_{h\in X_t}\vec w\cdot R(h)P^{\vec\pi}_\mu(h)$}
                &\mbox{(By (2))}\\
            &= \mbox{$\vec w\cdot \sum_{h\in X_t}R(h)P^{\vec\pi}_\mu(h)$}
                &\mbox{(Vect.\ algebra)}\\
            &= \mbox{$
                \vec w\cdot
                \left(
                    \sum_{h\in X_t}R(h)P^{\pi_i}_\mu(h)
                \right)_{i=1}^n
                $}
                &\mbox{(Def.\ \ref{vectorizationdefn} part 2)}\\
            &= \vec w\cdot V^{\vec\pi}_{\mu,t}.
                &\mbox{(Def.\ \ref{vectorizationdefn} part 3)}
    \end{align*}

    (4) Follows from (3) and Definition \ref{performancedefn} part 2.

    (5) Follows from (4) and Definition \ref{performanceaveragerdefn}.
\end{proof}

Parts 4--5 of Theorem \ref{maintheorem} show our mixture operation
(Definition \ref{maindefn}) avoids unexpected emergent behavior. Mixing
two agents in this way cannot result in a joint agent whose weaknesses
(or strengths) exceed the summed weaknesses (or strengths) of the original
two agents, at least as far as measured by expected performance in
RL environments.

\section{EQUIVALENCE OF WEAK AND STRONG SYMMETRY}
\label{symmetrysection}

In this section, we will investigate two symmetry properties
which a weighted intelligence measure might satisfy. A priori, one
property seems stricly stronger, but we will show
that in fact, they are equivalent. Throughout this section, we
assume that the background set $\mathcal R$ has the following
additional property:
whenever $\mathcal R$ contains any reward $r$, then $\mathcal R$
also contains $-r$.

\begin{definition}
\label{dualagentsdefn}
(Dual Agents)
\begin{enumerate}
    \item
    For each $h\in\mathcal H$,
    we define the \emph{dual} of $h$, denoted $\overline h$, to be
    the sequence obtained
    by replacing every percept $(o,r)$ in $h$ by $(o,-r)$ (in other words:
    replacing every reward $r$ in $h$ by $-r$).
    \item
    Suppose $\pi$ is an agent.
    We define the \emph{dual} of $\pi$, denoted $\overline \pi$, as follows:
    for each $h\in (\mathcal E\mathcal A)^*\mathcal E$,
    for each action $y\in\mathcal A$,
    $\overline\pi(y|h)=\pi(y|\overline h)$.
\end{enumerate}
\end{definition}

In plain language, $\overline\pi$ acts the way $\pi$ would act if $\pi$ wanted to
seek punishments and avoid rewards.

\begin{lemma}
\label{doublenegationlemma}
    If $x$ is any agent or history,
    then
    $\overline{\overline x}=x$.
\end{lemma}

\begin{proof}
    Trivial as $-(-r)=r$ for all real $r$.
\end{proof}

\begin{lemma}
\label{asteriskcommuteswithoverlinelemma}
    For any agent $\pi$ and any $h\in(\mathcal E\mathcal A)^*\mathcal E$,
    $P^\pi(\overline h)=P^{\overline{\pi}}(h)$.
\end{lemma}

\begin{proof}
    By induction on $h$.
\end{proof}

\begin{definition}
    An agent $\pi$ is \emph{self-dual} if $\overline{\pi}=\pi$.
\end{definition}

In plain language, self-dual agents only seek to extremize total reward, without
caring about the \emph{sign} of that total reward: a self-dual agent's actions do
not change if rewards and punishments are swapped. In particular, any reward-ignoring
agent is self-dual. It seems reasonable to expect that an intelligence measure
should assign intelligence $0$ to reward-ignoring agents. This motivates us to
consider intelligence-measure symmetries with respect to duality.
The main result in this section will be the equivalence of two
such symmetry conditions.
But first, we will use mixture agents to characterize
self-dual agents (up to equivalence modulo a natural equivalence relation).

\begin{definition}
\label{equivdefn}
    If $\pi$ and $\rho$ are agents, we say $\pi\equiv\rho$ if the
    following conditions hold:
    \begin{enumerate}
        \item For all $h\in\mathcal H$, $P^\pi(h)=0$ iff $P^\rho(h)=0$.
        \item For all $h\in(\mathcal E\mathcal A)^*\mathcal E$,
            if $P^\pi(h)\not=0$ then for all $y\in\mathcal A$,
            $\pi(y|h)=\rho(y|h)$.
    \end{enumerate}
\end{definition}

\begin{remark}
    Intuitively, Definition \ref{equivdefn} says that $\pi\equiv\rho$
    iff the histories which are ``possible'' for $\pi$ (in the sense of
    Remark \ref{impossibleremark}) are exactly the histories which are
    ``possible'' for $\rho$, and $\pi=\rho$ on those histories
    ($\pi$ and $\rho$ may differ on ``impossible'' histories).
\end{remark}

\begin{lemma}
\label{equivrelationlemma}
    $\equiv$ (Definition \ref{equivdefn}) is an equivalence
    relation.
\end{lemma}

\begin{proof}
    Straightforward.
\end{proof}

\begin{lemma}
\label{piopluspilemma}
    Let $\vec w=(w_1,\ldots,w_n)$ be positive reals,
    $w_1+\cdots+w_n=1$. For every agent $\pi$,
    $\pi\equiv\vec w\cdot (\pi,\ldots,\pi)$ (where
    $(\pi,\ldots,\pi)$ has length $n$).
\end{lemma}

\begin{proof}
    See Supplementary Materials.
\end{proof}

\begin{lemma}
\label{reflectionmakesjanuslemma}
    For any agent $\pi$,
    $(\frac12,\frac12)\cdot(\pi,\overline\pi)$ is self-dual.
\end{lemma}

\begin{proof}
    See Supplementary Materials.
\end{proof}

\begin{proposition}
\label{janusagentcharacterizationproposition}
    (Characterization of self-dual agents modulo $\equiv$)
    For any agent $\pi$, the following are equivalent:
    \begin{enumerate}
        \item $\pi\equiv\rho$ for some self-dual agent $\rho$.
        \item $\pi\equiv(\frac12,\frac12)\cdot(\rho,\overline{\rho})$
            for some agent $\rho$.
    \end{enumerate}
\end{proposition}

\begin{proof}
    ($\Rightarrow$)
    Assume $\pi\equiv\rho$ for some self-dual agent $\rho$.
    By Lemma \ref{piopluspilemma}
    $\rho\equiv (\frac12,\frac12)\cdot(\rho,\rho)$,
    but $\rho=\overline{\rho}$ by self-duality,
    so
    $\rho\equiv (\frac12,\frac12)\cdot(\rho,\overline{\rho})$.
    By transitivity of $\equiv$ (Lemma \ref{equivrelationlemma}),
    $\pi\equiv(\frac12,\frac12)\cdot(\rho,\overline{\rho})$.

    ($\Leftarrow$)
    Assume $\pi\equiv(\frac12,\frac12)\cdot(\rho,\overline{\rho})$ for some agent $\rho$.
    By Lemma \ref{reflectionmakesjanuslemma},
    $(\frac12,\frac12)\cdot(\rho,\overline{\rho})$ is self-dual.
\end{proof}

Proposition \ref{janusagentcharacterizationproposition} is analogous to the
fact that a function $f:\mathbb R\to\mathbb R$ is even
(i.e.\ satisfies $f(x)=f(-x)$)
iff $f(x)=\frac12(g(x)+g(-x))$ for some $g:\mathbb R\to\mathbb R$.

We will show (Theorem \ref{equivalentsymmetriescor}) that
the following two symmetry conditions are equivalent (improving
an informal result of \citet{alexander2021reward}).

\begin{definition}
    (Weighted intelligence measure symmetry properties)
    Let $\Upsilon$ be a weighted intelligence measure.
    \begin{enumerate}
        \item $\Upsilon$ is \emph{weakly symmetric} if
            $\Upsilon(\pi)=0$ for every self-dual agent $\pi$.
        \item $\Upsilon$ is \emph{strongly symmetric} if
            $\Upsilon(\overline\pi)=-\Upsilon(\pi)$ for every agent $\pi$.
    \end{enumerate}
\end{definition}

\begin{theorem}
\label{equivalentsymmetriescor}
    A weighted intelligence measure $\Upsilon$ is weakly
    symmetric iff it is strongly symmetric.
\end{theorem}

\begin{proof}
    (Weak $\Rightarrow$ Strong)
    Assume $\Upsilon$ is weakly symmetric.
    Let $\pi$ be any agent.
    By Lemma \ref{reflectionmakesjanuslemma},
    $(\frac12,\frac12)\cdot(\pi,\overline\pi)$ is self-dual.
    So by weak symmetry,
    $\Upsilon((\frac12,\frac12)\cdot(\pi,\overline\pi))=0$.
    Thus by Theorem \ref{maintheorem} (part 5),
    \[
        (\mbox{$\frac12$},\mbox{$\frac12$})\cdot\Upsilon((\pi,\overline\pi))
        =\mbox{$\frac12$}\Upsilon(\pi)+\mbox{$\frac12$}\Upsilon(\overline\pi)=0.
    \]
    So $\Upsilon(\overline{\pi})=-\Upsilon(\pi)$.
    By arbitrariness of $\pi$, $\Upsilon$ is strongly symmetric.

    (Strong $\Rightarrow$ Weak)
    Trivial.
\end{proof}

The canonical weighted intelligence measure is the universal
intelligence measure $\Upsilon_U$ of \citet{legg2007universal},
where each computable environment
$\mu$ is given weight $2^{-K(\mu)}$, where $K(\mu)$ is $\mu$'s Kolmogorov complexity.
This depends non-trivially on the background UTM $U$,
prompting \citet{leike2015bad} to ask: ``What are [...] desirable
properties of a UTM?''
UTMs formalize programming languages, so the question is equivalent
to ``What are desirable properties of a programming language?''\ (i.e., \emph{inherently}
desirable properties, as opposed to subjectively desirable properties like whether or
not white-space matters).
Intrinsically desirable UTM properties are elusive; attempts---such as
\citep{muller2010stationary}---to find them confirm the difficulty thereof.
In the RL context, symmetry conditions on $\Upsilon_U$ are candidate desirable properties
for the UTM: $U$ is weakly (resp.\ strongly) symmetric iff $\Upsilon_U$ is.
The equivalence of weak and strong symmetry provides some justification for
considering this UTM property to be inherently desirable (in the RL context).

\section{DISCERNABILITY AND SEPARABILITY}
\label{convexitysection}

In this section, we give another application of mixture agents.
We define natural notions of \emph{discernability} and \emph{separability}
for sets of agents,
and we give an interesting characterization of separability
in terms of discernability and mixtures.
Essentially, we import notions from convex geometry into reinforcement learning.
Below, if $\Pi$ is a set of agents, then
$\Pi^c$ is the set of agents $\rho$ such that $\rho\not\in\Pi$.

\begin{definition}
\label{discernabilitydefn}
    A set $\Pi$ of agents is \emph{discernable} if there exists
    an environment $\mu$ such that for all agents $\pi$, $\rho$:
    \begin{enumerate}
        \item $V^\pi_\mu$ and $V^\rho_\mu$ are defined.
        \item If $\pi\in\Pi$ and $\rho\in\Pi^c$, then $V^\pi_\mu\not=V^\rho_\mu$.
    \end{enumerate}
\end{definition}

Intuitively, $\Pi$ is discernable if there is some environment in which
no member of $\Pi$ has the same expected performance as any member of $\Pi^c$.

For the next definition, recall that a subset $I$ of the reals
is \emph{convex} if the following requirement holds:
for all real $i_1<i_2<i_3$, if $i_1\in I$ and $i_3\in I$, then $i_2\in I$.
Two sets are \emph{disjoint} if they have no point in common.

\begin{definition}
\label{incentivizabilitydefn}
    A set $\Pi$ of agents is \emph{separable} if there exists
    an environment $\mu$ and disjoint convex sets $I$ and $J$ of
    reals such that for every agent $\pi$:
    \begin{enumerate}
        \item
        $V^\pi_\mu$ is defined.
        \item
        If $\pi\in \Pi$ then $V^\pi_\mu\in I$.
        \item
        If $\pi\in \Pi^c$ then $V^\pi_\mu\in J$.
    \end{enumerate}
\end{definition}

Equivalently, $\Pi$ is separable if either: there is some environment where
every member of $\Pi$ outperforms every member of $\Pi^c$; or there is some
environment where every member of $\Pi$ underperforms every member of $\Pi^c$.
Clearly separability implies discernability, but what about the converse?
We will use mixtures to state a partial converse.

\begin{definition}
\label{mixtureclosuredef}
    A set $\Pi$ is \emph{closed under mixtures} if the following
    holds: for all $\vec w=(w_1,\ldots,w_n)$
    with $w_1+\cdots+w_n=1$,
    for all agents $\vec\pi=(\pi_1,\ldots,\pi_n)$,
    if $\pi_i\in \Pi$ for every $i=1,\ldots,n$, then
    $\vec w\cdot\vec\pi\in\Pi$.
\end{definition}

If we think of agents as points in space,
Definition \ref{mixtureclosuredef} is a universal mixture agent
analog of the higher-dimensional convexity notion from convex geometry.

\begin{lemma}
\label{convexsetlemma}
    For every environment $\mu$ and set $\Pi$ of agents,
    $
        S_{\Pi,\mu}=\{
            r\in\mathbb R
            \,:\,
            \exists \pi_1,\pi_2\in\Pi\mbox{ s.t. }V^{\pi_1}_\mu\leq r\leq V^{\pi_2}_\mu
        \}
    $
    is convex.
\end{lemma}

\begin{proof}
    See Supplementary Material.
\end{proof}

\begin{theorem}
\label{separabilitycharacterizationthm}
    (Characterization of Separability)
    For any set $\Pi$ of agents, the following are equivalent:
    \begin{enumerate}
        \item $\Pi$ is separable.
        \item $\Pi$ is discernable, and both $\Pi$ and $\Pi^c$ are closed under mixtures.
    \end{enumerate}
\end{theorem}

\begin{proof}
    ($1\Rightarrow 2$)
    Assume $\Pi$ is separable, and let $\mu,I,J$ be as in
    Definition \ref{incentivizabilitydefn}, so $I$ and $J$ are disjoint.

    To see $\Pi$ is discernable, let $\pi,\rho$ be any agents.
    By condition 1 of Definition \ref{incentivizabilitydefn},
    $V^\pi_\mu$ and $V^\rho_\mu$ are defined. And if $\pi\in\Pi$, $\rho\in\Pi^c$, then
    by conditions 2 and 3 of Definition \ref{incentivizabilitydefn},
    $\pi\in I$ and $\rho\in J$. Since $I$ and $J$ are disjoint,
    $V^\pi_\mu\not=V^\rho_\mu$.

    To see $\Pi$ is closed under mixtures, let $\vec w=(w_1,\ldots,w_n)$
    and $\vec\pi=(\pi_1,\ldots,\pi_n)$ be as in Definition \ref{mixtureclosuredef}
    and assume each $\pi_i\in\Pi$.
    By choice of $I$, each $V^{\pi_i}_\mu\in I$.
    By Theorem \ref{maintheorem},
    $V^{\vec w\cdot\vec\pi}_\mu=\vec w\cdot V^{\vec\pi}_\mu$.
    Thus $V^{\vec w\cdot\vec\pi}_\mu$ is a convex combination
    of $V^{\pi_1}_\mu,\ldots,V^{\pi_n}_\mu$, which are elements of $I$.
    Since $I$ is convex, it follows that $V^{\vec w\cdot\vec\pi}_\mu\in I$.
    Since $I$ and $J$ are disjoint, $V^{\vec w\cdot\vec\pi}_\mu\not\in J$,
    so by choice of $J$, $\vec w\cdot\vec\pi\in\Pi$, as desired.
    A similar argument shows $\Pi^c$ is closed under mixtures.

    ($2\Rightarrow 1$)
    Assume $\Pi$ is discernable and both $\Pi$ and $\Pi^c$ are closed under mixtures.
    Since $\Pi$ is discernable, there is some environment $\mu$
    as in Definition \ref{discernabilitydefn}.
    Let $I=S_{\Pi,\mu}$, $J=S_{\Pi^c,\mu}$ as in Lemma \ref{convexsetlemma},
    so $I$ and $J$ are convex.
    From the definition of $I$ and $J$, clearly $V^\pi_\mu\in I$ for all $\pi\in\Pi$
    and $V^\rho_\mu\in J$ for all $\pi\in\Pi^c$. It only remains to show $I$ and $J$
    are disjoint. Assume not. Then there is some $r\in I\cap J$.
    By definition of $I$, there are $\pi_1,\pi_2\in\Pi$ such that
    $V^{\pi_1}_\mu\leq r\leq V^{\pi_2}_\mu$, and by definition of $J$,
    there are $\rho_1,\rho_2\in\Pi^c$ such that $V^{\rho_1}_\mu\leq r\leq V^{\rho_2}_\mu$.
    By basic algebra, there is a real $\alpha\in [0,1]$
    such that $\alpha V^{\pi_1}_\mu + (1-\alpha)V^{\pi_2}_\mu=r$.
    Let $\pi=(\alpha,1-\alpha)\cdot (\pi_1,\pi_2)$.
    By Theorem \ref{maintheorem} (part 4),
    $V^\pi_\mu = \alpha V^{\pi_1}_\mu + (1-\alpha)V^{\pi_2}_\mu=r$.
    And since $\Pi$ is closed under mixtures, $\pi\in\Pi$.
    By identical reasoning using $V^{\rho_1}_\mu\leq r\leq V^{\rho_2}_\mu$,
    there exists some $\rho\in \Pi^c$ such that $V^\rho_\mu=r$.
    But since $\mu$ satisfies condition 2 of
    Definition \ref{discernabilitydefn}, 
    and $\pi\in\Pi$ and $\rho\in\Pi^c$, this forces
    $V^\pi_\mu\not=V^\rho_\mu$, absurd.
\end{proof}

\section{LOCAL EXTREMA AND LATTICE POINTS}
\label{extremasection}

\begin{definition}
\label{modifyagentatoneplace}
    If $\pi$ is an agent, $h_0\in(\mathcal E\mathcal A)^*\mathcal E$,
    and $m$ is a probability distribution on $\mathcal A$,
    we write $\pi^{h_0\mapsto m}$ for the function which is identical to $\pi$
    except that $m$ decides the action distribution for
    $h_0$, that is,
    \[
        \pi^{h_0\mapsto m}(y|h)
        =
        \begin{cases}
            \pi(y|h) &\mbox{if $h\not=h_0$,}\\
            m(y) &\mbox{if $h=h_0$.}
        \end{cases}
    \]
\end{definition}

In plain language, $\pi^{h_0\mapsto m}$ is the result of changing $\pi$'s output
$\pi(\cdot|h_0)$ to $m$, but otherwise leaving $\pi$ unchanged.

\begin{lemma}
    $\pi^{h_0\mapsto m}$ (as in Definition \ref{modifyagentatoneplace})
    is an agent.
\end{lemma}

\begin{proof}
    Trivial.
\end{proof}

\begin{definition}
\label{sumofdistros}
    Suppose $\vec m=(m_1,\ldots,m_n)$ are probability distributions on $\mathcal A$
    and $\vec w=(w_1,\ldots,w_n)$ are positive reals with
    $w_1+\cdots+w_n=1$. By $\vec w\cdot\vec m$ we mean the function
    on $\mathcal A$ defined by
    \[
        (\vec w\cdot\vec m)(y) = w_1m_1(y) + \cdots + w_nm_n(y).
    \]
\end{definition}

\begin{lemma}
\label{wcdotmisaprobabilitydistro}
    If $\vec m$, $\vec w$ are as in Definition \ref{sumofdistros}
    then $\vec w\cdot\vec m$ is a probability distribution on $\mathcal A$.
\end{lemma}

\begin{proof}
    See Supplementary Materials.
\end{proof}

\begin{definition}
    For any agent $\pi$, for any $h\in(\mathcal E\mathcal A)^*\mathcal E$,
    for any probability distributions $\vec m=(m_1,\ldots,m_n)$ on $\mathcal A$,
    let $\pi^{h\mapsto \vec m}=(\pi^{h\mapsto m_1},\ldots,\pi^{h\mapsto m_n})$.
\end{definition}

The following proposition shows that
for any particular history $h$ and agent $\pi$,
for any decomposition of $\pi(\cdot|h)$ into a weighted sum
of probability distributions $m_1,\ldots,m_n$,
$\pi$ has the same intelligence as the weighted mixture of the corresponding $n$ agents
$\pi^{h\mapsto \vec m}$.

\begin{proposition}
\label{longproposition}
    Let $\Upsilon$ be any weighted intelligence measure, let $\pi$ be any agent,
    and let $h\in(\mathcal E\mathcal A)^*\mathcal E$.
    Suppose $\vec m$ and $\vec w$ are as in Definition \ref{sumofdistros}.
    If $\vec w\cdot\vec m = \pi(\cdot|h)$, then
    $
        \Upsilon(\pi)
        =
        \Upsilon(\vec w\cdot \pi^{h\mapsto \vec m}).
    $
\end{proposition}

\begin{proof}
    See Supplementary Materials.
\end{proof}

\begin{definition}
    Suppose $\pi$ and $\pi_1,\ldots,\pi_n$ are agents and $\Upsilon$ is a
    weighted intelligence measure.
    We say $\pi\succ_\Upsilon \pi_1,\ldots,\pi_n$
    if both:
    \begin{enumerate}
        \item $\Upsilon(\pi)\geq \Upsilon(\pi_i)$ for each $i=1,\ldots,n$; and
        \item $\Upsilon(\pi)>\Upsilon(\pi_i)$ for some $i=1,\ldots,n$.
    \end{enumerate}
    We define $\pi\prec_\Upsilon\pi_1,\ldots,\pi_n$ likewise (change $\geq$/$>$ to $\leq$/$<$).
\end{definition}

\begin{proposition}
\label{pointwisegenericnessthm}
    Let $\Upsilon$ be any weighted intelligence measure and let
    $\pi$ be an agent.
    Let $h\in (\mathcal E\mathcal A)^*\mathcal E$.
    For any probability distributions $\vec m=(m_1,\ldots,m_n)$ on $\mathcal A$,
    for any positive reals $\vec w=(w_1,\ldots,w_n)$ with $w_1+\cdots+w_n=1$,
    if $\vec w\cdot\vec m=\pi(\cdot|h)$,
    then $\pi\not\succ_\Upsilon \pi^{h\mapsto m_1},\ldots,\pi^{h\mapsto m_n}$
    and $\pi\not\prec_\Upsilon \pi^{h\mapsto m_1},\ldots,\pi^{h\mapsto m_n}$.
\end{proposition}

\begin{proof}
    If $\pi\succ_\Upsilon\pi^{h\mapsto m_1},\ldots,\pi^{h\mapsto m_n}$
    then this implies
    \begin{align*}
        \Upsilon(\pi)
            &= \Upsilon(\vec w\cdot\pi^{h\mapsto\vec m})
                &\mbox{(Proposition \ref{longproposition})}\\
            &= \vec w\cdot\Upsilon(\pi^{h\mapsto\vec m})
                &\mbox{(Theorem \ref{maintheorem})}\\
            &< w_1\Upsilon(\pi)+\cdots+w_n\Upsilon(\pi)
                &\mbox{(Assumption)}\\
            &= \Upsilon(\pi),
                &\mbox{($w_1+\cdots+w_n=1$)}
    \end{align*}
    absurd. Similar reasoning holds for $\prec_\Upsilon$.
\end{proof}

\begin{definition}
\label{metricdefn}
    (Local intelligence extrema)
    \begin{enumerate}
    \item
        We make the space of all agents into a metric space by defining
        the distance from agent $\pi$
        to agent $\rho$ to be
        $
            d(\pi,\rho)
            =
            \sup_{h\in(\mathcal E\mathcal A)^*\mathcal E,y\in\mathcal A}\left|
                \pi(y|h) - \rho(y|h)
            \right|.
        $
    \item
        Suppose $\Upsilon$ is a weighted intelligence measure. An agent $\pi$
        is a \emph{strict local maximum} (resp.\ \emph{strict local minimum})
        of $\Upsilon$ if there is some real $\epsilon>0$
        such that for every agent $\rho\not\equiv\pi$
        (recall Definition \ref{equivdefn}), if $d(\rho,\pi)<\epsilon$
        then $\Upsilon(\pi)>\Upsilon(\rho)$ (resp.\ $\Upsilon(\pi)<\Upsilon(\rho)$).
        If $\pi$ is a strict local maximum or minimum of $\Upsilon$ then
        $\pi$ is a \emph{strict local extremum} of $\Upsilon$.
    \end{enumerate}
\end{definition}

\begin{definition}
\label{deterministicinpracticedefn}
    An agent $\pi$ is \emph{deterministic in all possible histories}
    if the following condition
    holds. For all $h\in(\mathcal E\mathcal A)^*\mathcal E$,
    if $P^\pi(h)\not=0$ then for all $y\in\mathcal A$,
    $\pi(y|h)\in\{0,1\}$.
\end{definition}

Note that an agent $\pi$ can be deterministic in all possible histories
(Definition \ref{deterministicinpracticedefn}) and still assign a
probability $0<\pi(y|h)<1$, provided $P^\pi(h)=0$ (recall Remark \ref{impossibleremark}).
It is easy to show that $\pi$ is deterministic in all possible histories iff $\pi\equiv\rho$
for some strictly
deterministic $\rho$ (i.e., some $\rho$ such that $\rho(y|h)\in\{0,1\}$ for
all $y,h$).

Theorem \ref{extremitythm} below sheds light on the geometry of RL agent intelligence.
By considering agent $\pi$'s coordinates to be
the values $\pi(y|h)$
for all $y$, $h$, we can view agents as inhabiting
infinite-dimensional Euclidean space.
An agent $\pi$ is a \emph{lattice point} (i.e., a point with integer coordinates)
iff $\pi$ is strictly deterministic.
We can picture $z=\Upsilon(\pi)$ as a ``surface'' above
agent-space
(more precisely: a surface above agent-space in some places,
below it in others, and which intersects it in the ``$z$-intercept''
$\Upsilon(\pi)=0$).
Theorem \ref{extremitythm} says that, modulo $\equiv$,
$z=\Upsilon(\pi)$
cannot have any ``hyperridges'' or ``hypertroughs'' above non-lattice points.

\begin{theorem}
\label{extremitythm}
    (``Strict local extrema are deterministic'')
    For any weighted intelligence measure $\Upsilon$, for any agent $\pi$,
    if $\pi$ is a strict local extremum of $\Upsilon$, then
    $\pi$ is deterministic in all possible histories.
\end{theorem}

\begin{proof}
    Assume $\pi$ is not deterministic in all possible histories, so there exist
    $h\in(\mathcal E\mathcal A)^*\mathcal E$ and $y_0\in\mathcal A$
    such that $P^\pi(h)\not=0$ and $0<\pi(y_0|h)<1$.
    We will show $\pi$ is not a strict local maximum (a similar argument
    shows $\pi$ is not a strict local minimum) of $\Upsilon$.
    Since $0<\pi(y_0|h)<1$ and $\pi(\cdot|h)\in\Delta\mathcal A$,
    there must be some $y_1\in\mathcal A$, $y_1\not=y_0$, such that
    $0<\pi(y_1|h)<1$. Let $\epsilon>0$.
    Since $0<\pi(y_0|h)<1$ and $0<\pi(y_1|h)<1$, it follows
    that there is some $0<\epsilon'\leq \epsilon$
    such that $0<\pi(y_0|h)\pm\epsilon'<1$ and $0<\pi(y_1|h)\pm\epsilon'<1$.
    Define $m_1,m_2:\mathcal A\to \mathbb R$ by
    \[
        m_i(y) = \begin{cases}
            \pi(y|h)+(-1)^i\epsilon' &\mbox{if $y=y_0$,}\\
            \pi(y|h)-(-1)^i\epsilon' &\mbox{if $y=y_1$,}\\
            \pi(y|h) &\mbox{otherwise.}
        \end{cases}
    \]
    By choice of $\epsilon'$ it follows that $m_1,m_2\in\Delta\mathcal A$.
    Let $\vec w=(\frac12,\frac12)$, $\vec m=(m_1,m_2)$;
    clearly $\vec w\cdot\vec m=\pi(\cdot|h)$.
    By Prop.\ \ref{pointwisegenericnessthm},
    $\pi\not\succ_\Upsilon \pi^{h\mapsto m_1},\pi^{h\mapsto m_2}$,
    thus $\Upsilon(\pi)\leq \Upsilon(\pi^{h\mapsto m_i})$
    for some $i\in\{1,2\}$.
    Clearly $\pi\not\equiv \pi^{h\mapsto m_i}$
    and $d(\pi,\pi^{h\mapsto m_i})=\epsilon'\leq\epsilon$.
    By arbitrariness of $\epsilon$,
    $\pi$ is not a strict local maximum of $\Upsilon$.
\end{proof}

In a sense, the above proof is a lack-of-emergent-behavior proof.
The non-determinacy of $\pi$ allows us to perturb $\pi$ very slightly
in two opposing directions, in such a way that $\pi$ is the weighted
mixture of the two perturbations. If, say, both perturbations strictly
reduced $\pi$'s intelligence, then their mixture would exhibit
emergent behavior (namely: ``behave at least as intelligently
as $\pi$'') of a type ruled out by Theorem \ref{maintheorem}.

We have worked in this section using the metric of Definition \ref{metricdefn}
for simplicity. Similar reasoning would apply to various other metrics
as well.

\section{UNIVERSAL MIXTURE-ENVIRONMENTS}
\label{mixtureenvsection}

\begin{definition}
\label{stronglywellbehaveddefn}
    An environment $\mu$ is \emph{strongly well-behaved} if
    $\mu$ is well-behaved and for all agents $\pi$ and all
    $t\in\mathbb N$, $-1\leq V^\pi_{\mu,t}\leq 1$.
    A weighted intelligence measure $\Upsilon$ is \emph{strongly well-behaved}
    if corresponding weights $\{w_\mu\}_{\mu\in W}$
    (as in Definition \ref{performanceaveragerdefn}) exist such that
    $\sum_{\mu\in W}w_\mu=1$ and such that
    $w_\mu=0$ for all $\mu$ not strongly well-behaved (informally: the weights
    underlying $\Upsilon$ sum to $1$, and any environment not strongly
    well-behaved has weight $0$).
\end{definition}

If $\mu$ never gives negative rewards, then
$-1\leq V^\pi_{\mu,t}\leq 1$ is equivalent to $-1\leq V^\pi_\mu\leq 1$.
Thus if the reward-space $\mathcal R$ is $\subseteq [0,1]$
(as in \citep{legg2007universal}),
then every nonzero weighted intelligence measure
is a constant multiple of a strongly well-behaved one.
We will show (Theorem \ref{universalenvtheorem})
that for every strongly well-behaved $\Upsilon$,
there is an environment $\mu_\Upsilon$ such that for all agents $\pi$,
$\Upsilon(\pi)=V^\pi_{\mu_\Upsilon}$.

In this section, let $\mathscr W$ be the set infinite sequences
$\vec w=(w_1,w_2,\ldots)$ with each $w_i>0$ real and
$\sum_{i=1}^\infty w_i=1$. For any $w\in\mathscr W$ and any
bounded sequence $\vec v=(v_1,v_2,\ldots)$, we define the dot
product $\vec w\cdot\vec v=\sum_{i=1}^\infty w_iv_i$ (the boundedness
of $\vec v$ implies this sum converges).

\begin{definition}
\label{infinitedimensionalvectorizationdefn}
    (Compare Definition \ref{vectorizationdefn})
    Let $\vec w\in\mathscr W$, $\pi$ an agent, $\mu$ an environment,
    $h\in\mathcal H$, $t\in\mathbb N$,
    and $\vec\mu=(\mu_1,\mu_2,\ldots)$ an infinite sequence of environments.
    Define:
    \begin{itemize}
        \item
            $P_{\vec\mu}(h)=(P_{\mu_1}(h),P_{\mu_2}(h),\ldots)$.
        \item
            $P^\pi_{\vec\mu}(h)=(P^\pi_{\mu_1}(h),P^\pi_{\mu_2}(h),\ldots)$.
        \item
            $V^\pi_{\vec\mu,t}=(V^\pi_{\mu_1,t},V^\pi_{\mu_2,t},\ldots)$.
        \item
            $V^\pi_{\vec\mu}=(V^\pi_{\mu_1},V^\pi_{\mu_2},\ldots)$ if
            every $V^\pi_{\mu_i}$ is defined.
    \end{itemize}
\end{definition}

\begin{definition}
\label{mixtureenvdefn}
    (Mixture environments---compare Definition \ref{maindefn})
    Assume $\vec w\in\mathscr W$ and $\vec\mu=(\mu_1,\mu_2,\ldots)$ is an infinite
    sequence of environments. Define an environment $\vec w\cdot\vec\mu$ by:
    \[
        (\vec w\cdot \vec\mu)(x|h)
        =
        \begin{cases}
            \dfrac{\vec w\cdot P_{\vec\mu}(hx)}{\vec w\cdot P_{\vec\mu}(h)}
            &\mbox{if $\vec w\cdot P_{\vec\mu}(h)\not=0$,}\\
            1/|\mathcal E| &\mbox{otherwise.}
        \end{cases}
    \]
\end{definition}

\begin{lemma}
\label{wcdotSisenvlemma}
    (Compare Lemma \ref{mixturereallyisanagent})
    $\vec w\cdot\vec\mu$ (as in Definition \ref{mixtureenvdefn})
    is indeed an environment.
\end{lemma}

\begin{proof}
    See Supplementary Materials.
\end{proof}

To prove Lemma \ref{envmaintheorem} below, we will use
Tannery's Theorem, a result from real analysis \citep{bromwich2005introduction}.

\begin{lemma}
\label{tannerysthm}
    (Tannery's Theorem)
    Let $\{a_i:\mathbb N\to\mathbb R\}_{i=1}^\infty$ be a sequence of sequences
    such each $\lim_{t\to\infty}a_i(t)$ converges.
    Assume $\{w_i\}_{i=1}^\infty$ satisfies
    $\sum_{i=1}^\infty w_k<\infty$ and for all $i>0$,
    for all $t\in\mathbb N$,
    $|a_i(t)|\leq w_k$. Then
    \[
        \lim_{t\to\infty}\sum_{i=1}^\infty a_i(t)
        =
        \sum_{i=1}^\infty\lim_{t\to\infty} a_i(t).
    \]
\end{lemma}

\begin{lemma}
\label{envmaintheorem}
    (Compare Theorem \ref{maintheorem})
    Let $\vec w\in\mathscr W$, let $\vec\mu=(\mu_1,\mu_2,\ldots)$ be
    a sequence of strongly well-behaved environments, and let $\pi$ be any agent. Then:
    \begin{enumerate}
        \item
        For all $h\in\mathcal H$,
        $P_{\vec w\cdot\vec\mu}(h)=\vec w\cdot P_{\vec\mu}(h)$.
        \item
        For all $h\in\mathcal H$,
        $P^\pi_{\vec w\cdot\vec\mu}(h)=\vec w\cdot P^\pi_{\vec\mu}(h)$.
        \item
        For all $t\in\mathbb N$,
        $V^\pi_{\vec w\cdot \vec\mu,t}=\vec w\cdot V^\pi_{\vec\mu,t}$.
        \item
        $V^\pi_{\vec w\cdot \vec\mu}=\vec w\cdot V^\pi_{\vec\mu}$.
    \end{enumerate}
\end{lemma}

\begin{proof}
    For (1)--(3), see Supplementary Materials. For (4):
    \begin{align*}
        V^\pi_{\vec w\cdot\vec\mu}
            &= \lim_{t\to\infty} V^\pi_{\vec w\cdot \vec\mu,t}
                &\mbox{(Definition \ref{performancedefn})}\\
            &= \lim_{t\to\infty} \vec w\cdot V^\pi_{\vec\mu,t}
                &\mbox{(By (3))}\\
            &= \lim_{t\to\infty} \sum_{i=1}^\infty w_i V^\pi_{\mu_i,t}.
                &\mbox{(Definition \ref{infinitedimensionalvectorizationdefn})}
    \end{align*}
    For each $i\geq 1$, define $a_i:\mathbb N\to\mathbb R$
    by $a_i(t)=w_iV^\pi_{\mu_i,t}$.
    Since $\mu_i$ is strongly well-behaved, each $-1\leq V^\pi_{\mu_i,t}\leq 1$.
    It follows that for all $t\in\mathbb N$, $|a_i(t)|\leq w_i$.
    Furthermore, $\sum_{i=1}^\infty w_i=1<\infty$
    by Definition of $\mathscr W$. Thus
    \begin{align*}
        &{} \lim_{t\to\infty} \sum_{i=1}^\infty w_iV^\pi_{\mu_i,t}\\
            &= \sum_{i=1}^\infty\lim_{t\to\infty} w_iV^\pi_{\mu_i,t}
                &\mbox{(Tannery's Theorem)}\\
            &= \sum_{i=1}^\infty w_i\lim_{t\to\infty} V^\pi_{\mu_i,t}
                &\mbox{(Algebra)}\\
            &= \sum_{i=1}^\infty w_iV^\pi_{\mu_i}
                &\mbox{(Definition \ref{performancedefn})}\\
            &= \vec w\cdot V^\pi_{\vec\mu}.
                &\mbox{(Definition \ref{infinitedimensionalvectorizationdefn})}
    \end{align*}
    So $V^\pi_{\vec w\cdot\vec\mu}=\vec w\cdot V^\pi_{\vec\mu}$.
\end{proof}

Just as Theorem \ref{maintheorem} (parts 4--5) shows that
Definition \ref{maindefn} provides a way to mix agents without emergent behavior,
in the same way, Lemma \ref{envmaintheorem} shows that Definition \ref{mixtureenvdefn}
provides a way to mix environments without emergent behavior.

\begin{theorem}
\label{universalenvtheorem}
    For any strongly well-behaved weighted intelligence measure $\Upsilon$,
    there is an environment $\mu_\Upsilon$ such that for every agent $\pi$,
    $\Upsilon(\pi)=V^\pi_{\mu_\Upsilon}$.
\end{theorem}

\begin{proof}
    Let $\vec\mu=(\mu_1,\mu_2,\ldots)$ enumerate all
    strongly well-behaved environments (a countable set since every
    well-behaved environment is Turing computable).
    Let $(w_\mu)_{\mu\in W}$ be as in Definition \ref{stronglywellbehaveddefn}.
    Let $\vec w=(w_{\mu_1},w_{\mu_2},\ldots)$,
    thus $\vec w\in\mathscr W$.
    Then for any agent $\pi$,
    \begin{align*}
        V^\pi_{\vec w\cdot \vec\mu}
            &= \vec w\cdot V^\pi_{\vec\mu}
                &\mbox{(Lemma \ref{envmaintheorem})}\\
            &= \sum_{i=1}^\infty w_i V^\pi_{\mu_i}
                &\mbox{(Definition \ref{infinitedimensionalvectorizationdefn})}\\
            &= \Upsilon(\pi),
                &\mbox{(Definition \ref{stronglywellbehaveddefn})}
    \end{align*}
    so $\mu_\Upsilon=\vec w\cdot \vec\mu$ works.
\end{proof}

\section{SUMMARY}

We introduced (Definition \ref{maindefn}) an operation
which takes a finite sequence of RL agents and a
finite sequence of weights, and which outputs a new agent, which can
be thought of as a weighted mixture agent, with the property that
in any environment,
``The expected reward of a weighted mixture is the weighted
average of the expected rewards''
(Theorem \ref{maintheorem} part 4). Thus if intelligence is measured in
terms of performance,
``The intelligence of a weighted mixture is the weighted average
of the intelligences'' (Theorem \ref{maintheorem} part 5).
This construction enabled us to prove a number of results about
the geometry of RL agent intelligence measures, namely, results about
intelligence symmetry (Theorem \ref{equivalentsymmetriescor}),
convexity (Theorem \ref{separabilitycharacterizationthm}),
and strict local extrema (Theorem \ref{extremitythm}).
Finally, by applying the same mixture idea to environments instead of
agents, we established (Theorem \ref{universalenvtheorem})
that for a large class of performance-based
intelligence measures, there exist universal mixture environments,
i.e., environments in which every agent's total expected reward in
fact equals the agent's intelligence according to said measure.

\section*{Acknowledgements}

This work has been supported in parts by ARC grant DP150104590.

\bibliographystyle{apalike} 
\bibliography{main}

\appendix
\onecolumn

\section{SUPPLEMENTARY MATERIAL}

Here, we present detailed proofs missing from the main text due to length limit.

\subsection{Proof of Lemma \ref{factorizationlemma}}
\begin{proof}
    By induction on $h$.

    Case 1: $h=\varepsilon$. Then the lemma is trivial.

    Case 2: $h=gx$ for some $x\in\mathcal E$.
        Then
        \begin{align*}
            P^\pi_\mu(h)
                &= P^\pi_\mu(g)\mu(x|g)
                    &\mbox{(Definition \ref{pullbackdef})}\\
                &= P^\pi(g)P_\mu(g)\mu(x|g)
                    &\mbox{(Induction)}\\
                &= P^\pi(h)P_\mu(g)\mu(x|g)
                    &\mbox{(Definition \ref{pullbackdef})}\\
                &= P^\pi(h)P_\mu(h).
                    &\mbox{(Definition \ref{pullbackdef})}
        \end{align*}

    Case 3: $h=gy$ for some $y\in\mathcal A$.
        Similar to Case 2.
\end{proof}

\subsection{Proof of Lemma \ref{mixturereallyisanagent}}
\begin{proof}
    Let $h\in(\mathcal E\mathcal A)^*\mathcal E$.
    Clearly $(\vec w\cdot\vec\pi)(y|h)\geq 0$ for all $y\in\mathcal A$.
    It remains to show
    $\sum_{y\in\mathcal A}(\vec w\cdot\vec\pi)(y|h)=1$.

    Case 1: $\vec w\cdot {P^{\vec\pi}}(h)=0$. Then
    each $(\vec w\cdot\vec\pi)(y|h)=1/|\mathcal A|$ so the
    claim is immediate.

    Case 2: $\vec w\cdot {P^{\vec\pi}}(h)\not=0$. Then
    \begin{align*}
        &{} \sum_{y\in\mathcal A}(\vec w\cdot\vec\pi)(y|h)\\
            &= \sum_{y\in\mathcal A}
                \frac{\vec w\cdot {P^{\vec\pi}}(hy)}{\vec w\cdot {P^{\vec\pi}}(h)}
                &\mbox{(Definition \ref{maindefn})}\\
            &= \sum_{y\in\mathcal A}
                \frac
                {\vec w\cdot(P^{\pi_1}(hy),\ldots,P^{\pi_n}(hy))}
                {\vec w\cdot {P^{\vec\pi}}(h)}
                &\mbox{(Definition \ref{vectorizationdefn})}\\
            &= \sum_{y\in\mathcal A}
                \frac
                {w_1 P^{\pi_1}(h)\pi_1(y|h)+\cdots+w_n P^{\pi_n}(h)\pi_n(y|h)}
                {\vec w\cdot {P^{\vec\pi}}(h)}
                &\mbox{(Definition \ref{pullbackdef})}\\
            &= \frac{
                w_1 P^{\pi_1}(h)\left(\mbox{$\sum_{y\in\mathcal A}\pi_1(y|h)$}\right)
                +\cdots+
                w_n P^{\pi_n}(h)\left(\mbox{$\sum_{y\in\mathcal A}\pi_n(y|h)$}\right)
                }
                {\vec w\cdot {P^{\vec\pi}}(h)}
                &\mbox{(Algebra)}\\
            &= \frac
                {w_1 P^{\pi_1}(h)\cdot 1 + \cdots + w_n P^{\pi_n}(h)\cdot1}
                {\vec w\cdot {P^{\vec\pi}}(h)}
                =\frac{\vec w\cdot {P^{\vec\pi}}(h)}{\vec w\cdot {P^{\vec\pi}}(h)}=1.
                &\mbox{($\pi_i$ are agents)}
    \end{align*}
\end{proof}

\subsection{Proof of Lemma \ref{piopluspilemma}}
\begin{proof}
    Recall that the real numbers satisfy the so-called \emph{null-factor law}:
    for all real numbers $a$ and $b$, if $ab=0$, then $a=0$ or $b=0$.
    In other words, the product of two nonzero real numbers can never be zero.

    Write $\vec\pi$ for $(\pi,\ldots,\pi)$.
    We prove conditions 1 and 2 of Definition \ref{equivdefn}
    simultaneously by induction on $h$.
 
    Case 1: $h=\varepsilon$. Then
    $P^\pi(h)=\vec w\cdot{P^{\vec\pi}}(h)=1\not=0$, so
    vacuously $P^\pi(h)=0$ iff $\vec w\cdot{P^{\vec\pi}}(h)=0$
    (proving condition 1).
    For condition 2, there is nothing to check, since
    $\varepsilon\not\in(\mathcal E\mathcal A)^*\mathcal E$.

    Case 2: $h=h_0y_0$ for some
        $h_0\in(\mathcal E\mathcal A)^*\mathcal E$, $y_0\in\mathcal A$.
        For condition 2, there is nothing to prove, since
        $h\not\in(\mathcal E\mathcal A)^*\mathcal E$.
        For condition 1, we consider two cases.

        Subcase 2.1: $P^\pi(h_0)=0$.
        By induction, condition 1 holds for $h_0$, so
        $\vec w\cdot{P^{\vec\pi}}(h_0)=0$.
        By Definition \ref{pullbackdef},
        $P^\pi(h)=P^\pi(h_0)\pi(y_0|h_0)=0$
        and $P^{\vec w\cdot\vec\pi}(h)=0(\vec w\cdot\vec\pi)(y_0|h_0)=0$.
        So $P^\pi(h)=0$ iff $P^{\vec w\cdot\vec\pi}(h)=0$.

        Subcase 2.2: $P^\pi(h_0)\not=0$.
        Then
        \begin{align*}
            P^{\vec w\cdot \vec\pi}(h)
                &= \vec w\cdot{P^{\vec\pi}}(h)
                    &\mbox{(Theorem \ref{maintheorem})}\\
                &= w_1 P^\pi(h)+\cdots+w_n P^\pi(h)
                    &\mbox{(Def.\ of $\vec w$ and $\vec\pi$)}\\
                &= P^\pi(h)
                    &\mbox{($w_1+\cdots+w_n=1$)}\\
                &= P^\pi(h_0)\pi(y_0|h_0).
                    &\mbox{(Definition \ref{pullbackdef})}
        \end{align*}
        Since $P^\pi(h_0)\not=0$,
        by the null-factor law,
        it follows that
        $P^{\vec w\cdot\vec\pi}(h)=0$ iff $P^\pi(h)=0$ iff $\pi(y_0|h_0)=0$.

    Case 3: $h=h_0x$ for some $h_0\in (\mathcal E\mathcal A)^*$,
        $x\in\mathcal E$.
        By induction, conditions 1 and 2 hold for $h_0$.
        By Definition \ref{pullbackdef},
        $P^\pi(h)=P^\pi(h_0)$ and
        $P^{\vec w\cdot\vec\pi}(h)=P^{\vec w\cdot\vec\pi}(h_0)$,
        so condition 1 for $h$ follows.

        For condition 2,
        assume $P^\pi(h)\not=0$ and let $y\in\mathcal A$.
        By choice of $\vec w$ and $\vec\pi$,
        $\vec w\cdot{P^{\vec\pi}}(h)=w_1P^\pi(h)+\cdots+w_nP^\pi(h)=P^\pi(h)$.
        So, since $P^\pi(h)\not=0$, $\vec w\cdot{P^{\vec\pi}}(h)\not=0$.
        Thus
        \begin{align*}
            (\vec w\cdot\vec\pi)(y|h)
                &= \frac{\vec w\cdot {P^{\vec\pi}}(hy)}{\vec w\cdot{P^{\vec\pi}}(h)}
                    &\mbox{(Definition \ref{maindefn})}\\
                &= \frac{w_1P^\pi(hy)+\cdots+w_nP^\pi(hy)}
                    {w_1P^\pi(h)+\cdots+w_nP^\pi(h)}
                    &\mbox{(Def.\ of $\vec w$ and $\vec\pi$)}\\
                &= \frac{w_1P^\pi(h)+\cdots+w_nP^\pi(h)}{w_1P^\pi(h)+\cdots+w_nP^\pi(h)}
                    \pi(y|h)
                    &\mbox{(Definition \ref{pullbackdef})}\\
                &= \pi(y|h).
        \end{align*}
\end{proof}

\subsection{Proof of Lemma \ref{reflectionmakesjanuslemma}}
\begin{proof}
    Let $\vec w=(\frac12,\frac12)$.
    For any $h\in(\mathcal E\mathcal A)^*\mathcal E$ and $y\in\mathcal A$,
    we claim
    \[
        \overline{\vec w\cdot(\pi,\overline{\pi})}(y|h)
        =(\vec w\cdot(\pi,\overline{\pi}))(y|h).
    \]
    Noting that
    $\vec w\cdot P^{(\pi,\overline{\pi})}(h)
    =\frac12P^\pi(h)+\frac12P^{\overline\pi}(h)$
    and
    $\vec w\cdot P^{(\pi,\overline{\pi})}(\overline h)
    =\frac12P^\pi(\overline h)+\frac12P^{\overline\pi}(\overline h)$,
    Lemmas \ref{doublenegationlemma} and \ref{asteriskcommuteswithoverlinelemma}
    imply that
    \[\vec w\cdot P^{(\pi,\overline{\pi})}(h)
    =\vec w\cdot P^{(\pi,\overline{\pi})}(\overline h).\]
    So if
    $\vec w\cdot P^{(\pi,\overline{\pi})}(\overline h)=0$
    then $\vec w\cdot P^{(\pi,\overline{\pi})}(h)=0$
    and it follows from
    Definition \ref{maindefn}
    and Lemma \ref{asteriskcommuteswithoverlinelemma} that
    $\overline{\vec w\cdot(\pi,\overline{\pi})}(y|h)
    =(\vec w\cdot(\pi,\overline{\pi}))(y|h)=1/|\mathcal A|$.
    So assume $\vec w\cdot P^{(\pi,\overline{\pi})}(\overline h)\not=0$.
    Then:
    \begin{align*}
        \overline{\vec w\cdot(\pi,\overline{\pi})}(y|h)
        &= (\vec w\cdot(\pi,\overline{\pi}))(y|\overline h)
            &\mbox{(Definition \ref{dualagentsdefn})}\\
        &= \frac
            {\frac12P^\pi(\overline hy)+\frac12P^{\overline\pi}(\overline hy)}
            {\frac12P^\pi(\overline h)+\frac12P^{\overline\pi}(\overline h)}
            &\mbox{(Definition \ref{maindefn})}\\
        &= \frac
            {\frac12P^\pi(\overline{hy})+\frac12P^{\overline\pi}(\overline{hy})}
            {\frac12P^\pi(\overline h)+\frac12P^{\overline\pi}(\overline h)}
            &\mbox{(Clearly $\overline hy=\overline{hy}$)}\\
        &= \frac
            {\frac12P^{\overline\pi}(hy)+\frac12P^{\overline{\overline\pi}}(hy)}
            {\frac12P^{\overline\pi}(h)+\frac12P^{\overline{\overline\pi}}(h)}
            &\mbox{(Lemma \ref{asteriskcommuteswithoverlinelemma})}\\
        &= (\vec w\cdot(\overline{\overline\pi},\overline\pi))(y|h)
            &\mbox{(Definition \ref{maindefn})}\\
        &= (\vec w\cdot(\pi,\overline\pi))(y|h).
            &\mbox{(Lemma \ref{doublenegationlemma})}
    \end{align*}
\end{proof}

\subsection{Proof of Lemma \ref{convexsetlemma}}

\begin{proof}
    Assume $i_1<i_2<i_3$ are reals with $i_1,i_3\in S_{\Pi,\mu}$; we must show
    $i_2\in S_{\Pi,\mu}$.
    Since $i_1\in S_{\Pi,\mu}$, there exist agents $\pi_1,\pi_2\in\Pi$
    such that $V^{\pi_1}_\mu\leq i_1\leq V^{\pi_2}_\mu$.
    And since $i_3\in S_{\Pi,\mu}$, there exist agents $\rho_1,\rho_2\in\Pi$
    such that $V^{\rho_1}_\mu\leq i_3\leq V^{\rho_2}_\mu$.
    Then $\pi_1,\rho_2\in \Pi$ satisfy
    $V^{\pi_1}_\mu\leq i_2\leq V^{\rho_2}_\mu$, showing $i_2\in S_{\Pi,\mu}$
    as desired.
\end{proof}

\subsection{Proof of Lemma \ref{wcdotmisaprobabilitydistro}}
\begin{proof}
    Clearly for every $y\in\mathcal A$,
    $(\vec w\cdot\vec m)(y) = w_1m_1(y) + \cdots + w_nm_n(y)$ is a nonnegative
    real. It remains to show $\sum_{y\in\mathcal A}(\vec w\cdot\vec m)(y)=1$.
    We compute:
    \begin{align*}
        &{} \sum_{y\in\mathcal A}(\vec w\cdot\vec m)(y)\\
        &=
        \sum_{y\in\mathcal A} w_1m_1(y) + \cdots + w_nm_n(y)
            &\mbox{(Definition \ref{sumofdistros})}\\
        &=
        w_1\left(\sum_{y\in\mathcal A} m_1(y)\right)
        + \cdots + w_n\left(\sum_{y\in\mathcal A}m_n(y)\right)
            &\mbox{(Basic Algebra)}\\
        &= w_1 + \cdots + w_n
            &\mbox{($m_1,\ldots,m_n$ are probability distr's)}\\
        &= 1.
    \end{align*}
\end{proof}

\subsection{Proof of Proposition \ref{longproposition}}
The following auxiliary lemmas will be used in our proof of
Proposition \ref{longproposition}.

\begin{lemma}
\label{firsttechlemmaforgenericity}
    Suppose $\pi$, $h_0$, $m$ are as in Definition \ref{modifyagentatoneplace}.
    Let $h\in\mathcal H$ be such that
    for every $y\in\mathcal A$,
    $h_0y$ is not an initial segment of $h$.
    Then $P^{\pi^{h_0\mapsto m}}(h)=P^\pi(h)$.
\end{lemma}

\begin{proof}
    By induction on $h$.
\end{proof}

\begin{lemma}
\label{thirdtechlemmaforgenericity}
    Suppose $\pi$, $h_0$, $m$ are as in Definition \ref{modifyagentatoneplace}.
    For any $y\in\mathcal A$,
    $P^{\pi^{h_0\mapsto m}}(h_0y)=P^\pi(h_0)m(y)$.
\end{lemma}

\begin{proof}
    Immediate by Definition \ref{pullbackdef} and Lemma \ref{firsttechlemmaforgenericity}.
\end{proof}

\begin{lemma}
\label{secondtechlemmaforgenericity}
    Suppose $\pi$, $h_0$, $m$ are as in Definition \ref{modifyagentatoneplace}.
    Assume $h\in\mathcal H$, $y_0\in\mathcal A$, and $h_0y_0$ is
    an initial segment of $h$. Assume $\pi(y_0|h_0)\not=0$. Then
    $P^{\pi^{h_0\mapsto m}}(h) = \frac{P^\pi(h)m(y_0)}{\pi(y_0|h_0)}$.
\end{lemma}

\begin{proof}
    By induction on $h$.

    Case 1: $h=h_0y_0$. Then
    \begin{align*}
        P^{\pi^{h_0\mapsto m}}(h)
        &= P^{\pi^{h_0\mapsto m}}(h_0)\pi^{h_0\mapsto m}(y_0|h_0)
            &\mbox{(Definition \ref{pullbackdef})}\\
        &= P^\pi(h_0)\pi^{h_0\mapsto m}(y_0|h_0)
            &\mbox{(Lemma \ref{firsttechlemmaforgenericity})}\\
        &= P^\pi(h_0)m(y_0)
            &\mbox{(Definition \ref{modifyagentatoneplace})}\\
        &= \frac{P^\pi(h_0)\pi(y_0|h_0)m(y_0)}{\pi(y_0|h_0)}
            &\mbox{(Basic Algebra)}\\
        &= \frac{P^\pi(h)m(y_0)}{\pi(y_0|h_0)}.
            &\mbox{(Definition \ref{pullbackdef})}
    \end{align*}

    Case 2: $h=h_0 y_0 h_1 x$ for some $h_1\in\mathcal H$
        and $x\in\mathcal E$. Then
    \begin{align*}
        P^{\pi^{h_0\mapsto m}}(h)
        &= P^{\pi^{h_0\mapsto m}}(h_0 y_0 h_1)
            &\mbox{(Definition \ref{pullbackdef})}\\
        &= \frac{P^\pi(h_0 y_0 h_1)m(y_0)}{\pi(y_0|h_0)}
            &\mbox{(Induction)}\\
        &= \frac{P^\pi(h_0 a_0 h_1 x)m(y_0)}{\pi(y_0|h_0)}
            &\mbox{(Definition \ref{pullbackdef})}\\
        &= \frac{P^\pi(h)m(y_0)}{\pi(y_0|h_0)}.
    \end{align*}

    Case 3: $h=h_0 y_0 h_1 y$ for some $h_1\in\mathcal H$ and
        $y\in\mathcal A$. Then
    \begin{align*}
        P^{\pi^{h_0\mapsto m}}(h)
        &= P^{\pi^{h_0\mapsto m}}(h_0 y_0 h_1)
            \pi^{h_0\mapsto m}(y|h_0 y_0 h_1)
            &\mbox{(Definition \ref{pullbackdef})}\\
        &= P^{\pi^{h_0\mapsto m}}(h_0 y_0 h_1)
            \pi(y|h_0 a h_1)
            &\mbox{(Definition \ref{modifyagentatoneplace})}\\
        &= \frac{P^\pi(h_0 y_0 h_1)\pi(y|h_0 y_0 h_1)m(y_0)}{\pi(y_0|h_0)}
            &\mbox{(Induction)}\\
        &= \frac{P^\pi(h_0 y_0 h_1 y)m(y_0)}{\pi(y_0|h_0)}
            &\mbox{(Definition \ref{pullbackdef})}\\
        &= \frac{P^\pi(h)m(y_0)}{\pi(y_0|h_0)}.
    \end{align*}
\end{proof}

\begin{proof}[Proof of Proposition \ref{longproposition}]
    \textbf{Subclaim:}
    For every $g\in\mathcal H$,
    $P^\pi(g)=P^{\vec w\cdot\pi^{h\mapsto \vec m}}(g)$.
    We prove this by induction on $g$.

    Case 1: $g=\varepsilon$.
    Then $P^\pi(g)=1
    =P^{\vec w\cdot\pi^{h\mapsto \vec m}}(g)$.

    Case 2: $g=fx$ for some $x\in\mathcal E$.
    Then
    \begin{align*}
        P^\pi(g)
            &= P^\pi(f)
                &\mbox{(Definition \ref{pullbackdef})}\\
            &= P^{\vec w\cdot\pi^{h\mapsto \vec m}}(f)
                &\mbox{(Induction)}\\
            &= P^{\vec w\cdot\pi^{h\mapsto \vec m}}(g).
                &\mbox{(Definition \ref{pullbackdef})}
    \end{align*}

    Case 3: $g=fy$ for some $y\in\mathcal A$.

    Subcase 3.1: $P^\pi(f)=0$.
    Then
    \begin{align*}
        P^{\vec w\cdot\pi^{h\mapsto \vec m}}(g)
            &= P^{\vec w\cdot\pi^{h\mapsto \vec m}}(f)
            (\vec w\cdot\pi^{h\mapsto \vec m})(y|f)
                &\mbox{(Definition \ref{pullbackdef})}\\
            &= P^\pi(f)(\vec w\cdot\pi^{h\mapsto \vec m})(y|f)
                &\mbox{(Induction)}\\
            &= 0.
    \end{align*}
    Similarly, $P^\pi(g)=0$. So $P^{\vec w\cdot\pi^{h\mapsto \vec m}}(g)=P^\pi(g)$.

    Subcase 3.2: $P^\pi(f)\not=0$ and $f=h$. Then:
    \begin{align*}
        P^{\vec w\cdot\pi^{h\mapsto \vec m}}(g)
            &= P^{\vec w\cdot\pi^{h\mapsto \vec m}}(hy)\\
            &= \vec w\cdot{P^{\pi^{h\mapsto \vec m}}}(hy)
                    &\mbox{(Theorem \ref{maintheorem})}\\
            &= w_1P^\pi(h)m_1(y)+\cdots+w_nP^\pi(h)m_n(y)
                    &\mbox{(Lemma \ref{thirdtechlemmaforgenericity})}\\
            &= P^\pi(h)\pi(y|h)
                    &\mbox{($\vec w\cdot\vec m=\pi(\cdot|h)$)}\\
            &= P^\pi(hy) = P^\pi(g).
                    &\mbox{(Definition \ref{pullbackdef})}
    \end{align*}

    Subcase 3.3: $P^\pi(f)\not=0$, $f\not=h$, and
    $f$ has an initial segment $h y_0$ ($y_0\in\mathcal A$).

    Then $\pi(y_0|h)\not=0$, lest
    we would have $P^\pi(f)=0$. Thus:
    \begin{align*}
        P^{\vec w\cdot\pi^{h\mapsto \vec m}}(g)
            &= P^{\vec w\cdot \pi^{h\mapsto \vec m}}(fy)\\
            &= \vec w\cdot{P^{\pi^{h\mapsto \vec m}}}(fy)
                    &\mbox{(Theorem \ref{maintheorem})}\\
            &= w_1\frac{P^\pi(fy)m_1(y_0)}{\pi(y_0|h)}
                +\cdots+w_n\frac{P^\pi(fy)m_n(y_0)}{\pi(y_0|h)}
                    &\mbox{(Lemma \ref{secondtechlemmaforgenericity})}\\
            &= \frac{P^\pi(fy)}{\pi(y_0|h)}(w_1m_1(y_0)+\cdots+w_nm_n(y_0))
                    &\mbox{(Basic Algebra)}\\
            &= \frac{P^\pi(fy)}{\pi(y_0|h)}\pi(y_0|h)
                    &\mbox{($\vec w\cdot\vec m=\pi(\cdot|h)$)}\\
            &= P^\pi(fy) = P^\pi(g).
    \end{align*}

    Subcase 3.4: $P^\pi(f)\not=0$, $f\not=h$, and $f$ has no initial segment
        of the form $hy_0$. Then:
    \begin{align*}
        P^{\vec w\cdot\pi^{h\mapsto \vec m}}(g)
            &= P^{\vec w\cdot\pi^{h\mapsto \vec m}}(fy)\\
            &= \vec w\cdot{P^{\pi^{h\mapsto \vec m}}}(fy)
                    &\mbox{(Theorem \ref{maintheorem})}\\
            &= w_1P^\pi(fy)+\cdots+w_nP^\pi(fy)
                    &\mbox{(Lemma \ref{firsttechlemmaforgenericity})}\\
            &= P^\pi(fy) = P^\pi(g),
                    &\mbox{($w_1+\cdots+w_n=1$)}
    \end{align*}
    as desired.

    This finishes the proof of the Subclaim.
    By Lemma \ref{factorizationlemma}, the Subclaim implies
    that for every well-behaved $\mu$ and every $g\in\mathcal H$,
    $
    P^\pi_\mu(g)
    =
    P^{\vec w\cdot\pi^{h\mapsto \vec m}}_\mu(g)
    $.
    By Definition \ref{performancedefn} (part 1), this implies that
    for every well-behaved $\mu$ and every $t\in\mathbb N$,
    $
        V^{\pi}_{\mu,t}
        =
        V^{\vec w\cdot \pi^{h\mapsto \vec m}}_{\mu,t}.
    $
    The proposition follows by Definition \ref{performancedefn} (part 2).
\end{proof}

\subsection{Proof of Lemma \ref{wcdotSisenvlemma}}
\begin{proof}
    Let $h\in(\mathcal E\mathcal A)^*$.
    Clearly $(\vec w\cdot \vec\mu)(x|h)\geq 0$ for all $x\in\mathcal E$.
    It remains to show $\sum_{x\in\mathcal E}(\vec w\cdot \vec\mu)(x|h)=1$.
    If $\vec w\cdot P_{\vec\mu}(h)=0$ then each $(\vec w\cdot \vec\mu)(x|h)=1/|\mathcal E|$
    so the claim is immediate; assume not. Then:
    \begin{align*}
        &{} \sum_{x\in\mathcal E}(\vec w\cdot \vec\mu)(x|h)\\
            &= \sum_{x\in\mathcal E}\frac{\vec w\cdot P_{\vec\mu}(hx)}{\vec w\cdot P_{\mu}(h)}
                &\mbox{(Definition \ref{mixtureenvdefn})}\\
            &= \sum_{x\in\mathcal E}
                \frac{
                    \sum_{i=1}^\infty w_i P_{\mu_i}(hx)
                }{
                    \sum_{i=1}^\infty w_i P_{\mu_i}(h)
                }
                &\mbox{(Definition \ref{infinitedimensionalvectorizationdefn})}\\
            &= \sum_{x\in\mathcal E}
                \frac{
                    \sum_{i=1}^\infty w_i P_{\mu_i}(h)\mu_i(x|h)
                }{
                    \sum_{i=1}^\infty w_i P_{\mu_i}(h)
                }.
                &\mbox{(Definition \ref{pullbackdef})}
    \end{align*}
    By absolute convergence, we can rearrange the order of summation without
    altering the sum, so the above is
    \[
        \frac{
            \sum_{i=1}^\infty w_i P_{\mu_i}(h)\sum_{x\in\mathcal E}\mu_i(x|h)
        }{
            \sum_{i=1}^\infty w_i P_{\mu_i}
        },
    \]
    and each $\sum_{x\in\mathcal E}\mu(x|h)=1$ since each $\mu\in\Delta\mathcal E$,
    so the whole fraction reduces to $1$.
\end{proof}

\subsection{Proof of Lemma \ref{envmaintheorem} (1--3)}
\begin{proof}
    (1) By induction on $h$.

    Case 1: $h=\varepsilon$. Then
    \begin{align*}
        \vec w\cdot P_{\vec\mu}(h)
        &= \sum_{i=1}^\infty w_i P_{\mu_i}(h)
            &\mbox{(Definition \ref{infinitedimensionalvectorizationdefn})}\\
        &= \sum_{i=1}^\infty w_i
            &\mbox{($P_{\mu_i}(\varepsilon)=1$)}\\
        &= 1
            &\mbox{(Definition of $\mathscr W$)}\\
        &= P_{\vec w\cdot \vec\mu}(h).
            &\mbox{(Definition \ref{pullbackdef})}
    \end{align*}

    Case 2: $h=gx$ for some $x\in\mathcal E$.

    Subcase 2.1: $\vec w\cdot P_{\vec\mu}(g)=0$.
        This means $\sum_{i=1}^\infty w_i P_{\mu_i}(g)=0$.
        Since each $w_i>0$, this implies each $P_{\mu_i}=0$.
        From this it easily follows that
        $P_{\vec w\cdot \vec\mu}(gx)=\vec w\cdot P_{\vec\mu}(gx)=0$.

    Subcase 2.2: $\vec w\cdot P_{\vec\mu}(g)\not=0$. Then
    \begin{align*}
        P_{\vec w\cdot \vec\mu}(h)
            &= P_{\vec w\cdot \vec\mu}(g)(\vec w\cdot\vec\mu)(x|g)
                &\mbox{(Definition \ref{pullbackdef})}\\
            &= \vec w\cdot P_{\vec\mu}(g)(\vec w\cdot\vec\mu)(x|g)
                &\mbox{(Induction)}\\
            &= \vec w\cdot P_{\vec\mu}(g)
                \frac{\vec w\cdot P_{\vec\mu}(gx)}{\vec w\cdot P_{\mu}(g)}
                &\mbox{(Definition \ref{mixtureenvdefn})}\\
            &= \vec w\cdot P_{\vec\mu}(gx) = \vec w\cdot P_{\vec\mu}(h).
    \end{align*}

    Case 3: $h=gy$ for some $y\in\mathcal A$.
    Then
    \begin{align*}
        P_{\vec w\cdot \vec\mu}(h)
            &= P_{\vec w\cdot \vec\mu}(g)
                &\mbox{(Definition \ref{pullbackdef})}\\
            &= \vec w\cdot P_{\vec\mu}(g)
                &\mbox{(Induction)}\\
            &= \sum_{i=1}^\infty w_i P_{\mu_i}(g)
                &\mbox{(Definition \ref{infinitedimensionalvectorizationdefn})}\\
            &= \sum_{i=1}^\infty w_i P_{\mu_i}(gy)
                &\mbox{(Definition \ref{pullbackdef})}\\
            &= \vec w\cdot P_{\vec\mu}(gy) = \vec w\cdot P_{\vec\mu}(h).
                &\mbox{(Definition \ref{infinitedimensionalvectorizationdefn})}
    \end{align*}

    (2) Compute:
    \begin{align*}
        P^\pi_{\vec w\cdot \vec\mu}(h)
            &= P^\pi(h) P_{\vec w\cdot \vec\mu}(h)
                &\mbox{(Lemma \ref{factorizationlemma})}\\
            &= P^\pi(h) \vec w\cdot P_{\vec\mu}(h)
                &\mbox{(By (1))}\\
            &= P^\pi(h) \sum_{i=1}^\infty w_i P_{\mu_i}(h)
                &\mbox{(Definition \ref{infinitedimensionalvectorizationdefn})}\\
            &= \sum_{i=1}^\infty w_iP^\pi(h)P_{\mu_i}(h)
                &\mbox{(Algebra)}\\
            &= \sum_{i=1}^\infty w_iP^\pi_{\mu_i}(h)
                &\mbox{(Lemma \ref{factorizationlemma})}\\
            &= \vec w\cdot P^\pi_{\vec\mu}(h).
                &\mbox{(Definition \ref{infinitedimensionalvectorizationdefn})}
    \end{align*}

    (3) Let $X_t$, $R$ be as in Definition \ref{performancedefn} and compute:
    \begin{align*}
        V^\pi_{\vec w\cdot\vec\mu,t}
            &= \sum_{h\in X_t}R(h)P^\pi_{\vec w\cdot\vec\mu}(h)
                &\mbox{(Definition \ref{performancedefn})}\\
            &= \sum_{h\in X_t}R(h)\vec w\cdot P^\pi_{\vec\mu}(h)
                &\mbox{(By (2))}\\
            &= \sum_{h\in X_t}R(h)\sum_{i=1}^\infty w_iP^\pi_{\mu_i}(h).
                &\mbox{(Definition \ref{infinitedimensionalvectorizationdefn})}
    \end{align*}
    Since $X_t$ is finite, this sum is absolutely convergent, so we can
    rearrange terms, and the sum is equal to
    \begin{align*}
        \sum_{i=1}^\infty w_i\sum_{h\in X_t}R(h)P^\pi_{\mu_i}(h)
            &= \sum_{i=1}^\infty w_i V^\pi_{\mu_i,t}
                &\mbox{(Definition \ref{performancedefn})}\\
            &= \vec w\cdot V^\pi_{\vec\mu,t}.
                &\mbox{(Definition \ref{infinitedimensionalvectorizationdefn})}
    \end{align*}
\end{proof}


\end{document}